\documentclass{article}

\usepackage[preprint]{neurips_2026}

\usepackage[utf8]{inputenc} 
\usepackage[T1]{fontenc}    
\usepackage{hyperref}       
\usepackage{url}            
\usepackage{booktabs}       
\usepackage{amsfonts}       
\usepackage{nicefrac}       
\usepackage{microtype}      
\usepackage{xcolor}         
\usepackage{multicol}

\usepackage{amsmath, amssymb, amsthm, bm}
\usepackage{graphicx}
\usepackage{pifont}     
\usepackage{hyperref}
\usepackage{algorithm}
\usepackage{algorithmic}
\usepackage{subcaption}
\usepackage{multirow}
\usepackage{enumitem}

\newtheorem{definition}{Definition}
\newtheorem{theorem}{Theorem}

\newcommand{\calM}{\mathcal{M}}

\newcommand{\until}[0]{\textbf{U}}
\newcommand{\eventually}[0]{\textbf{F}}
\newcommand{\always}[0]{\textbf{G}}
\newcommand{\embeddings}[0]{\mathcal{Z}}

\newcommand{\trace}[2]{\sigma_{#1}^{#2}}
\newcommand{\rob}[0]{\rho}
\newcommand{\bound}[0]{\textbf{b}}

\newcommand{\cmark}{\ding{51}}
\newcommand{\xmark}{\ding{55}}
\newcommand\tabletextsize{\small}

\title{Runtime Monitoring of
Perception-Based \\ Autonomous Systems via Embedding Temporal Logic}

\author{%
  Parv Kapoor\thanks{Indicates equal contribution} \\
  Software and Societal Systems Department \\
  Carnegie Mellon University \\
  \texttt{parvk@andrew.cmu.edu} \\
  \And
  Abigail Hammer$^*$ \\
  Software and Societal Systems Department \\
  Carnegie Mellon University \\
  \texttt{arhammer@andrew.cmu.edu} \\
  \AND
  Ashish Kapoor \\
  General Robotics \\
  \texttt{ashish@generalrobotics.company} \\
  \And
  Karen Leung \\
  Aeronautics and Astronautics Department \\
  University of Washington \\
  \texttt{kymleung@uw.edu} \\
  \And
  Eunsuk Kang \\
  Software and Societal Systems Department \\
  Carnegie Mellon University \\
  \texttt{eunsukk@andrew.cmu.edu} \\
}

\begin{document}

\maketitle

\begin{abstract}
Runtime monitoring of autonomous systems traditionally relies on mapping continuous sensor observations to discrete logical propositions defined over low-dimensional state variables. This abstraction breaks down in perception-driven settings, where such mappings require additional learned modules that are often computationally expensive, brittle, and semantically misaligned.
In this work, we propose \emph{Embedding Temporal Logic} (ETL), a temporal logic that performs monitoring directly in learned embedding spaces.
ETL defines predicates through distances between observed embeddings and target embeddings derived from reference observations. This formulation allows specifications to capture high-level perceptual concepts, such as similarity to visual goals or avoidance of semantic regions, that are difficult or impossible to express using traditional predicates.
By composing these predicates with temporal operators, ETL naturally expresses temporally extended and sequential perceptual behaviors. We introduce ETL monitors for evaluating specifications over bounded embedding traces, along with a conformal calibration procedure that provides reliable and safety-oriented predicate evaluation.
We evaluate our approach across multiple manipulation environments to show that ETL achieves strong empirical agreement with ground-truth semantics, including accurate monitoring of temporally composed behaviors.

\end{abstract}

\section{Introduction}
\label{sec:intro}

Modern autonomous systems, from self-driving vehicles to robotic manipulators, increasingly rely on learned representations for perception, prediction, and decision making~\citep{hafner2020dreamer,tdmpc2, zhou2025dinowm, kim24openvla, intelligence2025pi06vlalearnsexperience, ye2026worldactionmodelszeroshot}. These  representations allow autonomous systems to overcome the challenges of operating on explicit state space representations (such as object poses and velocities), which often require state estimation pipelines and auxiliary localization modules. We refer to systems that operate on learned representations as \emph{perception-based systems}. Perception-based systems map high-dimensional sensor streams such as images, video, or lidar to compact latent representations,
which are then consumed by downstream policies, planners, or world models \citep{nvidia2026worldsimulationvideofoundation, baniodeh2025scalinglawsmotionforecasting}. 

A promising approach to achieving high-assurance autonomy involves formally specifying desired properties of a system and applying techniques such as formal verification \citep{verification-survey2019} and runtime monitoring \citep{Maler2004MonitoringTP} to check whether the system satisfies these properties \citep{seshia-atva18}. In particular, \emph{runtime monitoring} has attracted  considerable interest as it can be deployed \emph{online} to provide rigorous guarantees about the system behavior without incurring the cost of an exhaustive offline analysis. A runtime monitor periodically evaluates the execution of a system and raises an alert when the system exhibits undesirable behavior~\citep{colombo2022runtime}. This ability to provide lightweight but rigorous online assurance has led to successful applications in domains such as autonomous vehicles~\citep{schon2026spatiotemporal}, drones~\citep{gu2023successful}, and robotic manipulators~\citep{8836933}.  

Runtime monitoring relies on the availability of \emph{formal specifications} that capture the desired properties of a system. For systems with low-dimensional state representations, specification notations such as \textit{signal temporal logic (STL)} \citep{Maler2004MonitoringTP} provide an expressive formalism to specify behavioral properties over logical \emph{predicates}. Each predicate encodes a condition over a state variable that can be evaluated as true or false at each step of execution, e.g., whether its position, velocity, or force is below or above a given threshold. 

However, for perception-based systems, writing such formal specifications over learned representations remains an open challenge \citep{seshia-atva18}. For these systems, low-dimensional state representations are often unavailable or require specialized, ad-hoc perception modules.
For example, translating a property such as ``the robot is near the obstacle'' or ``the gripper is holding the object'' into predicates requires either an additional  classifier or detector, or a handcrafted feature extractor tailored to the task (e.g., to detect whether the concept of the gripper ``holding'' an object is present in the current scene) \citep{doi:10.1177/02783649231223546}. Adding these modules can introduce new sources of brittleness, calibration errors, and domain dependence into the system. Worse, whenever the vocabulary of concepts for specification evolves (e.g., to also be able to express properties about the gripper ``dropping'' an object), it may be necessary to augment the existing perception modules or add new ones to support the new concepts.
Overall, there is a \emph{fundamental mismatch} between (i) the latent space over which typical perception systems operate and (ii) the low-dimensional state space over which specifications in existing temporal logic notations are expressed.

In this paper, we propose a new approach for formally specifying and monitoring the behavioral properties of perception-based autonomous systems. The key idea is to employ \emph{embeddings}, pretrained vector representations of observations, as a first-class concept in specifications, and express a property in
terms of distances between a \emph{target embedding} (an ideal representation of real-world concepts that the system interacts with) and an \emph{observed embedding} (a representation
generated by an encoder from a sensor observation during system execution).  The key insight is that \emph{pretrained encoders already embed semantic proximity in geometry} \citep{radford2021learning, oquab2024dinov2learningrobustvisual}: observations of semantically similar scenes map to nearby vectors in latent space. This makes perceptual properties directly expressible as geometric predicates; for instance, ``being near the obstacle'' can be represented as ``$\|z_t - z_{\mathrm{obstacle}}\|_2$ is small,'' where $z_{\mathrm{obstacle}}$ is an encoder representation of a reference image of that obstacle and $z_{\mathrm{t}}$ is an encoding of the current scene image. 
An expressive temporal logic specification can then be constructed by combining multiple embedding-based predicates and used as part of a runtime monitor to ensure that the system satisfies its desired property (e.g., ``If the gripper is holding an object, it will not drop the object until it is moved to a deposit box''). 

Although this idea is conceptually simple, making it useful for formal specification poses multiple challenges. First, there is the question of how target embeddings are generated: they can come from a reference image,  a demonstration, or a set of both, and different choices can induce meaningfully different predicates. Additionally, embeddings are learned, continuous, and model-dependent representations, and geometric proximity is not guaranteed to align perfectly with the logical distinctions required for monitoring. A central challenge, therefore, is to turn embedding-space similarity into a well-defined specification primitive: one must decide which geometric relationships correspond to predicate satisfaction, how to calibrate decision thresholds, and how these predicates are composed to create  system specifications. These issues make embedding-based specifications substantially more challenging than simply reusing learned features inside an existing monitor.

In this paper, we make the following four contributions: (i) we introduce \emph{Embedding Temporal Logic} (ETL), a temporal logic for specifying perceptual behaviors directly over observations (Section~\ref{sec:etl_semantics}); (ii) we formally define Boolean satisfaction semantics over bounded embedding traces, yielding an online monitor for perceptual specifications (Sections~\ref{sec:etl_semantics} and~\ref{sec:monitor_def}); (iii) we propose data-driven methods for calibrating embedding predicate thresholds, making them feasible for safety-oriented monitoring (Section~\ref{sec:threshold}); and (iv) we evaluate ETL-based monitors across navigation and manipulation domains, showing that they can faithfully monitor atomic and sequential perceptual behaviors across diverse environments (Section~\ref{sec:experiments}).

\section{Background and Related Work}
\label{sec:background}

\paragraph{Formal Specifications for Robotic Systems}
Temporal logics, such as Linear Temporal Logic (LTL), STL, and Metric Temporal Logic (MTL) have been used to formally verify  complex behaviors in cyber-physical and robotic systems.
These logics have been used for trajectory planning \citep{kress2009temporal, Sun2022MultiagentMP, leungbackpropagation}, reinforcement learning \citep{aksaray2016qlearning, alur2023policy, 10160953}, runtime monitoring \citep{bartocci_specificationbased_2018}, and adaptive control \citep{MPC-STL, doi:10.1146/annurev-control-053018-023717, Lindemann2019ControlBF, kapoor2025stlcg++}.
These logics can struggle with systems that rely on ML for perception, where input data can have a variable number of objects in frame and evolving bounding boxes.
Recently, Spatiotemporal Perception Logic (STPL) \citep{doi:10.1177/02783649231223546} was introduced, which combined Timed Quality Temporal Logic \citep{dokhanchi2018evaluating} with spatial logic and allows quantification over objects, as well as 2D and 3D spatial reasoning.

\paragraph{Pretrained Vision Encoders}
Recent advances in representation learning have produced pretrained vision encoders that are sufficiently expressive to serve as general-purpose perceptual representations across a wide range of visual domains \citep{oquab2024dinov2learningrobustvisual}. As a result, these models provide a practical foundation for extending temporal logic beyond low-dimensional state space representations. Pretrained vision encoders such as CLIP \citep{radford2021learning} and DINOv2 \citep{oquab2024dinov2learningrobustvisual} provide a shared embedding space in which perceptual similarity can be measured, making them a natural choice for defining specification predicates over observations.

\paragraph{Specification Based Runtime Monitoring}
Given a logical specification that is well-defined over bounded traces and encodes a desired system property, a runtime monitor is an online evaluation at each time step during an execution to assess whether the execution satisfies the given specification   \citep{Maler2004MonitoringTP, bartocci_specificationbased_2018}.
In practice, runtime monitoring often goes beyond Boolean verdicts and
employs quantitative semantics, such as \emph{robustness} measures in temporal
logic, which provide a real-valued signal indicating how much a trace
satisfies or violates the specification \citep{robustness}. These quantitative monitors are
particularly useful in continuous and stochastic systems, as they enable
graded feedback and can be integrated with optimization or control
algorithms for real-time decision making.

\paragraph{Conformal Prediction}
Conformal prediction \citep{10.5555/1062391} is a distribution-free calibration framework that turns a held-out calibration set into a finite-sample statistical guarantee under only an \textit{exchangeability} assumption. The exchangeability of data means that the calibration and test samples are identically distributed and order-independent. At a high level, it is a method for calibrating a model so that its predictions on unseen examples come with a reliability guarantee. 
We employ conformal prediction theory for calibrating thresholds for ETL predicates in Section \ref{sec:threshold}.

\section{Embedding Temporal Logic}
\label{sec:etl}

\begin{figure*}[t]
    \centering
    \includegraphics[width=\textwidth]{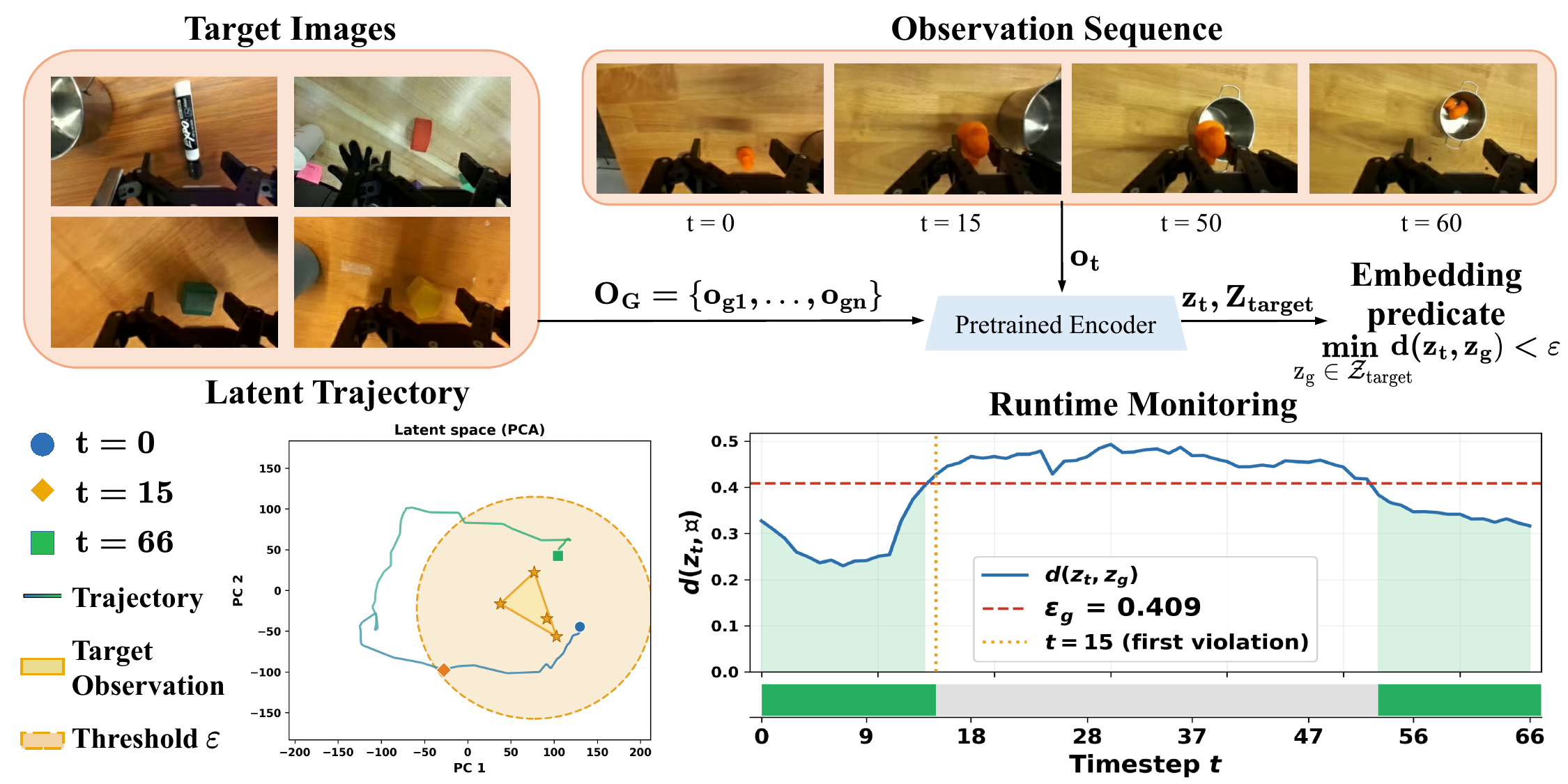}
\caption{
Overview of embedding-based runtime monitoring.
\textbf{Top:} Target observations $O_G=\{o_{g1},\dots,o_{gn}\}$ and the online observation $o_t$ are encoded by a pretrained vision encoder into target embeddings $z_G$ and the current embedding $z_t$. An embedding predicate then evaluates whether $z_t$ reaches the target set.
\textbf{Bottom left:} The embedding trajectory is projected onto its first two principal components, showing its evolution relative to the target embeddings.
\textbf{Bottom right:} Runtime monitoring computes the distance from $z_t$ to $z_G$ over time, thresholds it using $\epsilon$, and evaluates the temporal logic specification.
}
    \label{fig:etl_monitoring}
\end{figure*}

\subsection{Syntax and Semantics}
\label{sec:etl_semantics}

In our approach, a perception-based system is assumed to make an observation about the real world through a sensor (e.g., camera) at each step in its execution.
This observation is then passed through an encoder that translates the observation into an embedding. We formally model such a system as an \emph{Embedding Temporal Structure}.

\begin{definition}[Embedding Temporal Structure]\label{def:ets}
An \emph{Embedding Temporal Structure} is a tuple
\[
\calM \equiv (\mathcal{S}, \mathcal{O}, \embeddings{}, \phi_{obs}, \psi_{enc}, D_{\embeddings{}}, AP_z),
\]
where $\mathcal{S},\mathcal{O},\embeddings{}$ denote the sets of ground-truth states, observations, and
embedding spaces, respectively; $\phi_{obs} : \mathcal{S} \to \mathcal{O}$ is the observation function that maps the ground-truth states to the states that can be observed with the given sensor(s); $\psi_{enc}: \mathcal{O} \to \embeddings{}$ is the embedding function that converts an observation into an embedding;
$D_{\embeddings{}}$ is a set of admissible distance/similarity functions
$d:\embeddings{}\times\embeddings{}\to\mathbb{R}_{\ge0}$; and $AP_z$ is the set of embedding
predicates (Definition~\ref{def:etl_predicates}).
\end{definition}
Conceptually, $z$ is an approximation of a latent variable; i.e., the state of the world that is only indirectly observable by the system. 
The \emph{organization} of latent representations of observations within the embedding space $\embeddings{}$ is induced by the encoder $\psi_{enc}$ and the downstream models that subsequently consume this embedding for further tasks (e.g., object identification and prediction) and shape the embedding organization as part of their training.
This organization of embeddings in turn has a pronounced effect on the distances between embeddings. 

To reason about a system's temporal behavior, we must consider not only individual states 
but sequences of states over time. We therefore define an execution of the system and the corresponding 
embedding-space representation induced by the observation and encoder mappings.

\begin{definition}[Execution]\label{def:exec}
An execution of the system $\mathcal{M}$ is a finite or infinite sequence of
states
$
\varsigma = s_0, s_1, s_2, \ldots
$
such that each state $s_t$ belongs to the state space $\mathcal{S}$ for every
time index $t \in \mathbb{N}$.
\end{definition}

Each state in an execution can be mapped through $\phi_{obs}$ and $\psi_{enc}$,
yielding a corresponding representation in the embedding space.

\begin{definition}[Representation Map]\label{def:rep_map}
The mapping from the ground-truth states to embeddings is the \emph{representation map} $\eta \equiv  \psi_{enc} \circ  \phi_{obs}: S \to \embeddings{}$.
\end{definition}

\begin{definition}[Trace]\label{def:trace}
Given an execution $\varsigma = (s_i)_{i \in \mathbb{N}}$, the associated \emph{trace} is the sequence 
$\sigma = (z_i)_{i \in \mathbb{N}}$, where $z_i = \eta(s_i)$ for all $i \in \mathbb{N}$.
\end{definition}

\paragraph{Embedding Predicates}
Unlike classical temporal logics where atomic propositions are Boolean predicates over $S$, ETL predicates are defined over the embedding space  $\embeddings{}$.

\begin{definition}[Embedding Predicate] \label{def:etl_predicates}
An \emph{embedding predicate} $ap \in AP_z$ is a tuple
$
ap \equiv (\mathcal{Z}_{target}, d, \epsilon, \bowtie, a)
$
where: $\mathcal{Z}_{target} \subseteq \embeddings{}$ is a set of target embeddings;  $d \in D_{\embeddings{}}$, where $d: \embeddings{} \times \embeddings{} \rightarrow \mathbb{R}_{\ge 0}$, is a distance function; $\epsilon \in \mathbb{R}_{\geq 0}$ is a threshold; and $a$ is an aggregation operator (e.g., $\min$, $\max$).
\end{definition}

The set of target embeddings identifies sensor inputs that the system should either reach or avoid  in order for the system to either be considered safe or to have reached as a goal. 
Additional details for specifying target embeddings are discussed in Appendix \ref{sec:targets}.

We now define how an Embedding Predicate can be evaluated.

\begin{definition}[Predicate Satisfaction] \label{def:pred_sat}
    For a given embedding predicate $ap \in AP_z$ and embedding $z$, the evaluation of an embedding predicate is defined as:
\[
\delta_{ap}(z) = a(\{ d(z, z_g) \mid z_g \in \mathcal{Z}_{target} \})
\]
Then, the satisfaction of the predicate is the Boolean value determined by $\delta_{ap}(z) \bowtie \epsilon $.
\end{definition}

\paragraph{Example 1.}
Consider an embedding predicate $ap$ as illustrated in Figure~\ref{fig:etl_monitoring}, 
where the target images depict a desired visual concept (e.g.,  object not grasped) and are encoded into a target set 
$\mathcal{Z}_{target}$.
Given the embedding $z_t$ of the current observation, predicate evaluation measures the distance 
between $z_t$ and the target set. Let $d$ be the L2 distance, $a=\min$, and $\bowtie=\leq$. Then,
\[
\delta_{ap}(z_t)=\min_{z_g\in \mathcal{Z}_{target}} d(z_t,z_g).
\]
Using $a = \min$, we ensure that $\delta_{ap}(z_t)$ is the distance from the current observation to the closest target in the embedding space.
Then, the predicate is satisfied if $\delta_{ap}(z_t)\leq \epsilon$,
meaning that the current observation embedding lies sufficiently close to at least one target embedding, 
indicating that the desired visual concept is present.

Now that we have defined embedding predicates, we  next define the syntax and semantics of ETL to reason over sequences of embeddings over time.

\begin{definition}[ETL Syntax] \label{def:etl_syntax}
    The \emph{ETL syntax} of formula $\varphi$ is recursively defined as follows:
\begin{align*}
\varphi ::= ap \;|\; \neg \varphi \;|\; \varphi_1 \land \varphi_2 \;|\; \varphi_1 \until{} \varphi_2
\end{align*}
\end{definition}

As in LTL, the until operator
$\until{}$ can be used to express the \emph{eventually}
($\eventually{}$) operator and \emph{always} ($\always{}$) operator:
$\eventually{} \varphi = \texttt{True} \until{} \varphi$
and
$\always{} \varphi = \neg \eventually{} \neg \varphi$. 

\begin{definition}[ETL Semantics] \label{def:etl_semantics}
    For a given trace $\sigma$, the \emph{ETL semantics} for timestep $i$ are defined similarly to those of LTL, but over embedding traces: 
\begin{align*}
\centering
    &\trace{}{}, i \models ap \iff \delta_{ap}(z_i) \bowtie \epsilon \\
    &\trace{}{}, i \models \neg \varphi \iff \trace{}{}, i \not\models \varphi \\
    &\trace{}{}, i \models \varphi_1 \land \varphi_2 \iff \trace{}{}, i \models \varphi_1 \text{ and } \trace{}{}, i \models \varphi_2 \\
    &\trace{}{}, i \models \varphi_1 \until{} \varphi_2 \iff \exists j \geq i \text{ such that } \trace{}{}, j \models \varphi_2  \text{ and } \forall k \in [i,j), \trace{}{}, k \models \varphi_1
\end{align*}
\end{definition}

\paragraph{Robustness of ETL}
Certain types of temporal logics, such as STL \citep{Maler2004MonitoringTP}, allow for a quantitative notion of satisfaction, called \emph{robustness} \citep{robustness}, that represents the degree to which the system satisfies or violates a specification. To enable ETL to be used for quantitative monitoring, we also introduce a notion of robustness for ETL. For the sake of  brevity, the formal semantics for robustness are provided in Appendix \ref{app:quant}.

\paragraph{Example 2}
Reusing the embedding predicate $ap$ as defined in Example 1, consider the ETL specification
$
\varphi = \always(ap)
$,
which requires that the visual concept encoded by predicate $ap$ is always observed. Suppose the system generates the embedding trace
$
\trace{}{} = z_{12}, z_{13}, z_{14}, z_{15},
$
with minimum distances to $\mathcal{Z}_{target}$ as
$
[\delta_{ap}(z_t)]_{t=12}^{15}
=
[0.327,\; 0.374,\; 0.403,\; 0.427].
$

Using the robustness semantics for predicates,
$
\rho(ap, \trace{}{}, t, 3) = \epsilon - \delta_{ap}(z_t),
$
and setting $\epsilon = \epsilon_g = 0.409$, the predicate robustness at each timestep is:
$
[0.082,\; 0.035,\; 0.006,\; -0.019].
$
The robustness of the temporal specification is then
$
\rho(\varphi,\trace{}{},0,3)
=
\min_{t\in[0,3]} \rho(ap,\trace{}{},t,3)
=
-0.019.
$ Intuitively, although the system remains close to the object not being grasped for most of the window, the specification is violated because at timestep $15$ the embedding moves outside the threshold $\epsilon$, corresponding to the frame t=15 in Figure~\ref{fig:etl_monitoring} where the object is grasped.

\subsection{ETL Specification Monitors}
\label{sec:monitor_def}

An ETL monitor evaluates the satisfaction of a specification over an embedding trace for a given system.
ETL monitors are defined only over the finite trace of a system in order to be evaluated during the execution of the system, and thus are only defined for specifications that are well-defined over finite traces.

\begin{definition}[Finite Trace] \label{def:trace_finite}
    For a given trace $\trace{}{}$, a \emph{finite trace} $\trace{\le t}{}$ is a trace consisting of the first $t+1$ states  of $\trace{}{}$.
\end{definition}
We refer to a finite trace as a trace when clear from context. An ETL monitor is formally defined over a finite trace of embeddings 
by the following definition.

\begin{definition}[ETL Monitor] \label{def:etl_monitor}
For a trace $\trace{\le t}{}$, system $\calM$, and specification $\varphi$, an \emph{ETL Monitor} is
$M_\varphi(\trace{\le t}{})=(r^\varphi_0,\ldots,r^\varphi_t)\in\{-1,+1\}^{t+1}$,
where $r^\varphi_i=\operatorname{sgn}(\rho(\varphi,\trace{\le i}{},0,i))$ for each $0\le i\le t$, with $\operatorname{sgn}(x)=+1$ if $x\ge0$ and $-1$ otherwise.
\end{definition}

Intuitively, when the monitor is positive for a timestep, the system satisfies the ETL specification at that timestep. 
However, when the robustness of the system becomes negative at timestep $i$, the monitor raises an alert to indicate that the observed  execution violates $\varphi$ over the prefix $[0,i]$.

\paragraph{Semantic Correctness}
We denote an ETL monitor  to be \emph{semantically correct} iff it is equivalent with evaluation over ground-truth executions.  A formal definition of semantic correctness is provided in Appendix~\ref{appx:semantic_corr}. This notion of semantic correctness provides the basis for evaluating ETL monitors in practice. In our experiments (Section~\ref{sec:experiments}), we approximate semantic correctness by comparing the outputs of the ETL monitor against ground-truth monitors derived from state-based specifications.

\subsection{Constructing ETL Specifications in Practice}
\label{sec:practice}
Utilizing the formal semantics presented in Section \ref{sec:etl_semantics} requires several design choices for concrete instantiation: one must determine how target embeddings are specified, how observations are mapped into the representation space, which distance function is used to compare embeddings, and how predicate thresholds are calibrated in order to align with the intended semantic concept.
In practice, target embeddings should be selected by the engineer in the same format as the sensor input; e.g., for a camera, a target would be provided as a reference image.
The distance function selection  depends on the training objective of the encoder and ideally should align with the organization of the embedded space.
Additional discussion on these design decisions can be found in Appendix~\ref{app:practice}.

\section{Threshold Calibration for ETL Predicates}
\label{sec:threshold}
The preceding section described how ETL predicates are instantiated in practice through target embeddings, encoders, and distance functions. We now turn to threshold calibration. Recall from Definition~\ref{def:etl_predicates} that ETL predicates are evaluated by comparing embedding distances against a threshold~$\epsilon$. These thresholds play a critical role: they determine when a continuous similarity measure corresponds to predicate satisfaction.  
As a result, ETL specifications are not only parameterized by the encoder and distance function, but also by calibrated thresholds   
that define when a predicate holds.
We propose two approaches to calibrate thresholds.
For formal definitions of the thresholds, we refer readers to Appendix \ref{sec:appendix_threshold}.

\subsection{F1-Optimal Threshold}
We calibrate each predicate threshold on a held-out set, $\mathcal{D}_{cal}$, of trajectories with ground-truth labels per timestep. For each timestep $t$, we compute the embedding distance $d_t$ to the target embedding and, for a candidate threshold $\epsilon_1$, predict that the predicate holds whenever $d_t \le \epsilon_1$. We then search over possible threshold values based on the observed calibration distances and select the threshold, $\epsilon_{F1}$, that maximizes the F1 score with respect to the ground-truth labels.
Intuitively, this chooses the threshold that best matches the intended concept by balancing false positives and false negatives on the calibration data.

\subsection{Conformal Threshold with Recall Guarantee}
In safety-critical monitoring, missing a true positive event is often more costly than raising a false alarm. 
To bias threshold selection toward high recall, we compute an alternative threshold using \emph{split conformal prediction}, $\epsilon_{\mathrm{CP}}$ \citep{lei2017distributionfreepredictiveinferenceregression, 10.1561/2200000101}.   
We employ conformal prediction as it allows us to make no assumptions about the underlying distribution of embedding distances and provides a guarantee that transfers directly to deployment-time predicate evaluation.

In this work, we propose conformal \emph{ETL predicate calibration} by adapting conformal prediction to select predicate thresholds from embedding distance scores. 
Let $\varsigma^1,\dots,\varsigma^{n_{cal}}$ be calibration demonstrations, disjoint from training, with $z_g$ denoting the target latent embedding with a semantically corresponding ground-truth specification $\omega$.
For each calibration, $\varsigma^i$, a score is computed based on the maximum distance from the frame to the target such that the target satisfies $\omega$.
Then, for a user-given error level $\alpha\in(0,1)$, we sort the calibration scores, compute $k = \lceil(1-\alpha)(n_{cal}+1)\rceil$ and select $\epsilon_{CP}$ to be the k-th smallest score.
The formal definition and proof for conformal recall guarantees based on calibration demonstrations can be found in Appendix \ref{sec:appendix_threshold}.

 \section{Evaluation}
\label{sec:experiments}

We evaluate the semantic correctness of ETL monitors by comparing
embedding predicate outputs against ground-truth propositions. 
\footnote{Artifacts for reproducibility are provided at \url{https://github.com/ETLMonitoringAuthors/ETLMonitoring}} 
All experiments were run on a compute cluster using two NVIDIA GeForce RTX 5090 GPUs, each with 34.2 GB of memory. 

\subsection{Simple Navigation Dubins Car}
\label{sec:dubins}
First, we evaluate ETL specification based runtime monitoring in a two dimensional navigation task where privileged information about the state space is available, similar to the one proposed in   \citet{anysafe}. This fully observable setup allows direct comparison between embedding-based predicates and ground-truth specifications. 
\paragraph{Setup}
The navigation task is defined for a robot that respects discrete-time Dubins car dynamics in a controlled environment.
We generate $N=100$ trajectories using a feedback controller with obstacle avoidance. For each task, we construct pairs of equivalent specifications in the ground-truth state space, $\omega_i$, and in the embedding space, $\varphi_i$.
We utilize the encoder from the world model used in \citet{anysafe}, which is based on Dreamer \citep{Hafner2025}. The encoder produces a task-relevant representation space in which distances reflect semantic similarity between observations. We evaluate four specification patterns: \emph{Reach} ($\eventually A$), \emph{Avoid} ($\always \neg C$), \emph{Reach-Avoid} ($\eventually A \land \always \neg C$), and \emph{Sequential} ($\eventually(A \land \eventually B)$).
$A$, $B$, and $C$ denote reaching a goal in the top-right, the top-left, and the bottom-right corners of the environment respectively.
For visualization, see the goals in Figure~\ref{fig:dubins_traces}.
ETL monitoring is evaluated by comparing the satisfaction of each embedding-based specification $\varphi$ with the satisfaction of corresponding ground-truth specifications $\omega$ on the respective latent and state traces for each timestep of each trace.
We report precision, recall, and F1 to assess how closely ETL predicates match ground-truth events. For temporal specifications, we measure trajectory-level satisfaction agreement with the ground-truth monitor; for sequential specifications, we also report ordering accuracy of detected subgoals.
Additional details   can be found in Appendix~\ref{sec:appendix_dubins_setup}.

\paragraph{Results}
We now discuss our findings; a tabularized version of results can be found in Appendix~\ref{sec:appendix_dubins}.

\subparagraph{Can ETL monitors achieve semantic correctness as defined in Definition 15 with respect to ground-truth specifications in a controlled environment?} ETL shows strong 
{semantic correctness} for atomic predicates, visualized in Figure~\ref{fig:dubins_traces}. Across reach, avoid, and reach-avoid specifications, the F1-optimal monitor achieves F1 scores of $0.80$--$0.85$ with agreement above $96\%$, indicating that embedding predicates closely match the corresponding state-based events. This implies that embedding-space predicates can serve as faithful proxies for state-based propositions.
Additionally, for the sequential specification $A \rightarrow B$, the monitor achieves $100\%$ precision, recall, and ordering agreement at the episode level. Thus, ETL is not limited to detecting isolated semantic events; \emph{ETL can also correctly track their ordering over time.}

\subparagraph{Does threshold calibration produce the intended precision--recall tradeoff for safety-oriented monitoring?} Relative to the F1-optimal threshold, $\epsilon_{CP}$ increases recall for reach (from $0.83$ to $0.93$) while incurring only a modest drop in precision (from $0.87$ to $0.79$), with essentially no change in agreement. This shows that conformal calibration provides a practical safety-oriented operating point: it makes the monitor more conservative against missed detections without substantially changing overall semantic alignment.

Overall, the Dubins results show that ETL monitors align closely with state-based specification monitors, support monitoring of temporally composed behaviors, and exhibit a clear precision--recall tradeoff under threshold calibration.

\begin{figure*}[t]
    \centering

    \begin{subfigure}[t]{0.92\textwidth}
        \centering
        \includegraphics[width=\linewidth]{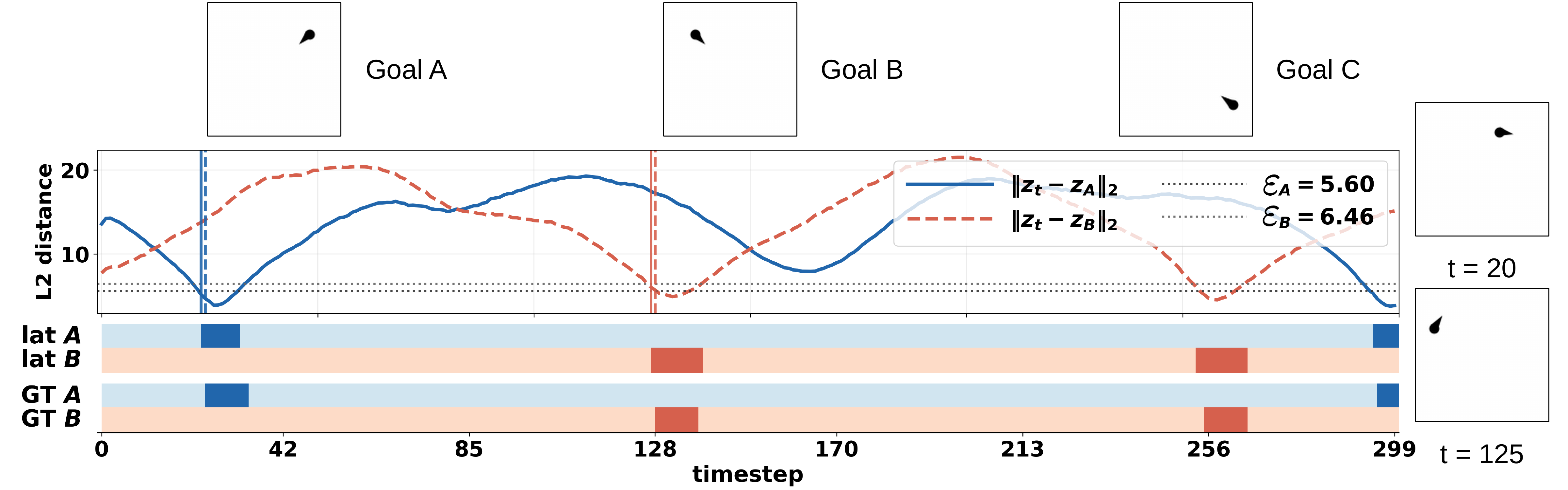}
        \caption{Boolean predicate traces on the Dubins car. Goals $A$, $B$, and $C$ are visualized along top of the image; the graph plots the L2 distance between embeddings and along with the F1-optimal thresholds. The bars along the bottom show when the latent and ground-truth predicates are satisfied (darker bars); two observation states are displayed on the left of the graph.}
        \label{fig:dubins_traces}
    \end{subfigure}

    \vspace{0.8em}

    \begin{subfigure}[t]{0.52\textwidth}
        \centering
        \includegraphics[width=\linewidth]{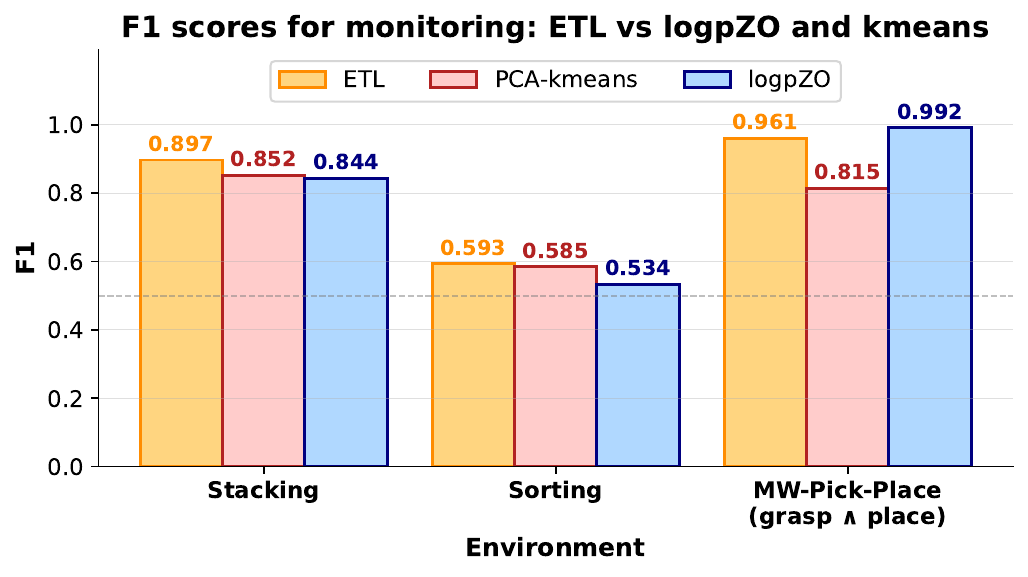}
        \caption{Predicate-monitoring F1 scores across domains.}
        \label{fig:f1_results}
    \end{subfigure}
    \hfill
    \begin{subfigure}[t]{0.47\textwidth}
        \centering
        \includegraphics[
            width=\linewidth,
            trim={0.1cm 0.2cm 0.4cm 0.2cm},
            clip
        ]{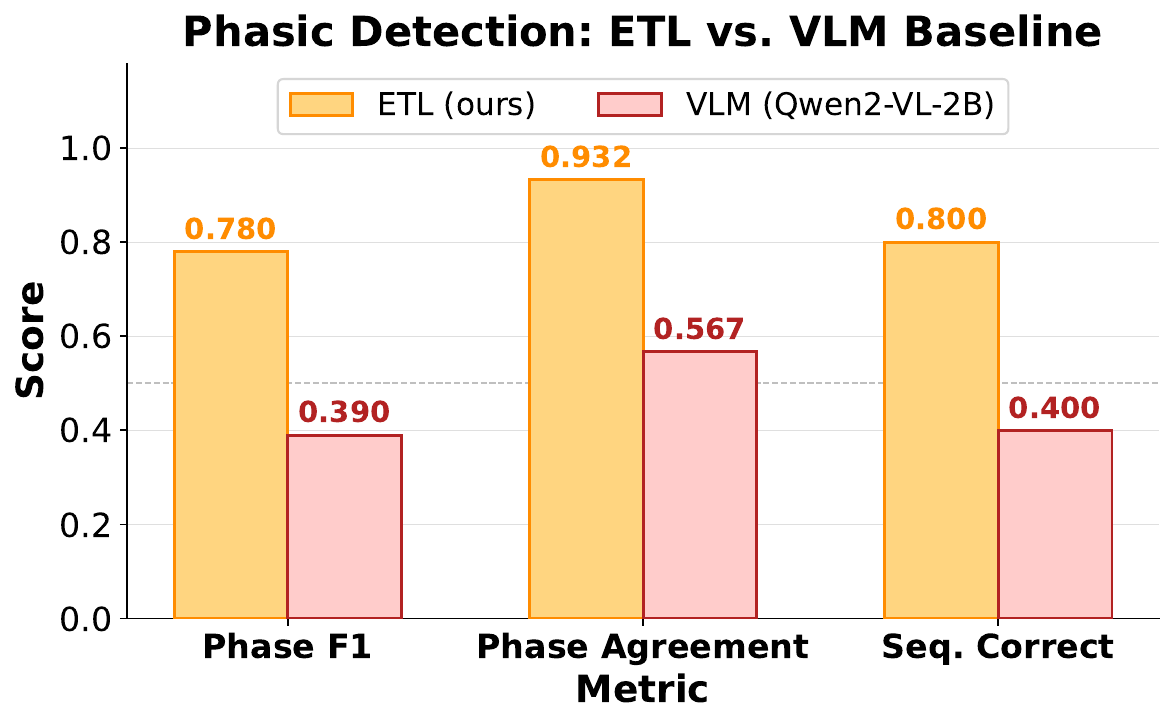}
        \caption{Comparison of ETL and Qwen2-VL-2B on phasic DROID episodes.}
        \label{fig:phasic_droid}
    \end{subfigure}

    \caption{
    Qualitative and quantitative ETL results.
    }
    \label{fig:etl_results_combined}
\end{figure*}
\subsection{Simulated Manipulation Tasks}
\label{sec:manip}
To assess whether ETL remains effective beyond the controlled navigation domain, we evaluate it for simulated contact-rich manipulation tasks with richer visual scenes, complex object interactions, and additional distractor objects that make perceptual monitoring harder. We consider two environments with complementary task structure: D3IL \citep{jia2024towards} and MetaWorld \citep{Hansen2025Newt}. For D3IL we consider two complex tasks: \textsc{Sorting} requires a Franka robot to push two blocks into their corresponding color-matching target boxes, while \textsc{Stacking} requires the robot to arrange colored blocks in a target region. For MetaWorld, we evaluate ETL on robotic-arm pick-and-place tasks with sequential grasp and place subgoals.  Unlike single-goal manipulation tasks, these benchmarks expose phase-like progress structures, making them a natural testbed for evaluating whether ETL predicates can monitor task-relevant semantic milestones in  rich manipulation settings. 

\paragraph{Setup}
We compare our ETL-based monitoring approach against recent embedding-based monitoring baselines: PCA-kmeans \citep{liu2024multitaskinteractiverobotfleet} and logpZO \citep{xu2023faildetect}. PCA-kmeans clusters principal components of successful observation embeddings and scores a new observation by its distance to the nearest cluster center, while logpZO fits a flow-matching density model over observation embeddings and uses the inferred latent norm as an uncertainty score. These baselines are the most direct comparison for ETL because they operate on observation embeddings and detect distributional deviations. We evaluate monitoring performance by comparing predicted satisfaction or violation labels against ground-truth labels derived from simulator rewards and state variables. Additional details regarding our experimental setup are provided in Appendix ~\ref{appx:tab:manip}.

\paragraph{Results: Can ETL outperform embedding-based monitoring baselines on simulated manipulation tasks?}
Figure~\ref{fig:f1_results} shows that ETL matches or exceeds the observation embedding baselines on average across the three simulated manipulation environments. Using $\epsilon_{F1}$, ETL achieves the highest average F1 score of $0.817$, compared with $0.790$ for logpZO and $0.751$ for PCA-kmeans. ETL improves over the strongest baseline on \textsc{Stacking} ($0.897$ vs. $0.852$) and slightly outperforms PCA-kmeans on \textsc{Sorting} ($0.593$ vs. $0.585$). \textsc{Sorting} is the hardest setting for all methods, suggesting that its OOD shifts induce more ambiguous changes in the embedding space. On MetaWorld pick-place-wall, logpZO reaches near-saturated performance ($0.992$ F1), indicating that the predicate boundary aligns closely with the support of successful embeddings. ETL remains competitive at $0.961$ F1, while PCA-kmeans lags behind at $0.815$.

\subsection{Real World  Evaluation via DROID}\label{sec:eval:droid}
Finally, we evaluate ETL on real-world robot data to assess its effectiveness for monitoring in-the-wild temporally extended manipulation behaviors.  
To this end, we utilize the manipulation dataset DROID  \citep{khazatsky2024droid}, which contains large-scale, diverse, and compositional manipulation demonstrations in more realistic settings.

\label{sec:real_world}

\paragraph{Setup}
We analyze manipulation episodes with compositional, multi-phase structure. We instantiate ETL specifications over phase-based predicates, where each phase corresponds to a contiguous interaction window in the demonstration. 
This setup allows us to evaluate how ETL monitoring performs over real-world manipulation episodes in comparison to a Vision Language Model (VLM)-based baseline monitor. For evaluation, we compute F1 score, agreement, and sequential ordering accuracy over phase-based predicates. For our baseline, we use a state-of-the-art VLM: Qwen2-VL \citep{wang2024qwen2vlenhancingvisionlanguagemodels}.
For each phase, the VLM is prompted with the ground-truth desired behavior and phase frame to determine if a task was completed; the response is distilled to a Boolean value, which is then used to compute frame-level F1, agreement, and sequential ordering accuracy for comparison with the ETL monitor.

\paragraph{Results: Can ETL accurately monitor perceptual specifications on real-world manipulation data in comparison to a  VLM-based baseline for monitoring phasic manipulation behaviors?}
ETL achieves a mean F1 score of $0.813$ and mean agreement of $0.940$, whereas the Qwen2-VL baseline reaches a mean F1 score of $0.390$ and mean agreement of $0.567$.
Though ETL and Qwen2-VL perform competitively on one of the five tasks, these results demonstrate that ETL generalizes to diverse, unstructured manipulation scenes, and is markedly more reliable overall, especially on phases that are visually similar but semantically distinct.
Additionally, ETL correctly identifies that sequential phasic ordering in $4/5$ episodes (Figure \ref{fig:phasic_droid}), indicating it can monitor compositional task structure beyond isolated phase detection. The single failure occurs in an episode with low inter-phase separability, suggesting that monitoring performance degrades when distinct phases occupy overlapping regions of the latent space.
Tabularized results can be found in Appendix~\ref{sec:appendix_droid}.

Overall, these results show that ETL extends beyond simulation to real-world manipulation data. ETL successfully monitors temporal and sequential task structure across demonstrations. Moreover, compared to a VLM-based baseline, ETL provides substantially more reliable monitoring of phasic behaviors without requiring explicit language supervision.

\section{Conclusion and Limitations}
\label{sec:conclusion}
In this work we propose a novel temporal logic (ETL) defined over the embedding space of pretrained encoders for perception-based autonomous systems.
Our experiments support three main conclusions.
First, embedding predicates can faithfully approximate ground-truth propositions across 
controlled navigation and manipulation tasks.
Second, ETL supports monitoring temporally extended properties: sequential specifications are accurately tracked across all evaluated environments.
Third, threshold calibration provides a meaningful operating tradeoff: the F1-optimal threshold yields the strongest overall alignment with ground truth, while the conformal threshold prioritizes high recall with a distribution-free guarantee.

The current approach has two main limitations: latent predicates are not yet fully interpretable in human-understandable terms, and monitoring performance depends on whether task-relevant semantic concepts are well separated in the encoder’s representation. These limitations point to promising directions for future work, including more transparent predicate explanations, encoder selection and adaptation for improved semantic separation  \cite{zarlenga2022conceptembeddingmodelsaccuracyexplainability}, temporal abstraction over subtask boundaries \cite{wang2026temporalstraighteninglatentplanning}, and adaptive online thresholding  \cite{areces2025online}.

\small
\bibliographystyle{plainnat}
\bibliography{bibs/etl_old, bibs/sources}

\appendix
\section{Constructing ETL Specifications in Practice}
\label{app:practice}

The formal semantics proposed in the preceding section define ETL as a specification language over embedding traces. 
To use ETL in practice, however, several design choices must be instantiated concretely: 
one must determine how target embeddings are specified, how observations are mapped into the representation space, 
which distance function is used to compare embeddings, and how predicate thresholds are calibrated so that they align with the intended semantic concept. In this section, we first outline how target embeddings are specified. Then, we highlight the different choices of encoders and distance functions. 

\subsection{Specifying Target Embeddings} \label{sec:targets}
ETL is specified using predicates over \emph{target embeddings} that correspond to parts of the physical world against which current observations are evaluated. 
In practice, the specifier (e.g., a system engineer) would not directly specify the mathematical representations of embeddings.
Instead, the targets would be provided in the same format as the sensor's input, e.g., for a camera, a target would be provided as a reference image, that is then translated into the target embeddings via the pretrained encoder.  This allows specifications such as \emph{“eventually reach a state similar to this image”} or \emph{“always avoid states resembling fire”} without requiring explicit symbolic labels.
By enabling specification directly over learned representations rather than explicit symbolic labels, our approach removes the need to manually define a finite predicate vocabulary from observations, a process that has traditionally depended on substantial expert knowledge and task-specific engineering in robotics. 

\subsection{Choice of Encoders and Distance functions}
The choice of distance function is another key practical consideration in evaluating ETL satisfaction. Distances over pretrained image embeddings capture perceptual similarity, but lack temporal context because these encoders are trained on static observations. World models mitigate this limitation by mapping observations into a latent space that evolves as a function of past states and actions. When trained with reconstruction-based objectives, such models often preserve much of the geometry of the original embedding space while enriching it with temporal structure. 
Consequently, similarity relationships defined over observation embeddings can often be carried over to the latent space in a meaningful way.

Accordingly, the choice of 
$d$ should align with the geometry induced by the representation. For instance, cosine distance is a natural choice for contrastive embeddings, whereas 
L2 distance is often more appropriate for reconstruction-based latent representations. More generally, the selected distance should be one that best reflects semantic similarity in the underlying space. We empirically study the effect of this design choice in Section~\ref{sec:experiments}.

\section{Quantitative semantics of ETL}
\label{app:quant}

\begin{definition}[ETL Robustness] \label{def:etl_quant}
For a trace $\trace{}{}$, the \emph{robustness of an ETL formula} at timestep $i$ is defined as follows:
\begin{eqnarray*}
  \rob(ap, \trace{}{}, i, \bound{}) & = & 
  \begin{cases}
    \epsilon - \delta_{ap}(z_i) & \bowtie \in \{\leq, <\} \\
    \delta_{ap}(z_i) - \epsilon &  \bowtie \in \{\geq, >\}
  \end{cases} \\
  \rob(\neg \varphi, \trace{}{}, i, \bound{}) & = & - \rob(\varphi, \trace{}{}, i, \bound{})\\
  \rob(\varphi_1 \land \varphi_2, \trace{}{}, i, \bound{}) & = &
  \min\big(\rob(\varphi_1, \trace{}{}, i, \bound{}), \rob(\varphi_2, \trace{}{}, i, \bound{})\big) \\
  \rob(\always{} \varphi, \trace{}{}, i, \bound{}) & = & \textbf{inf}_{k \in [i, \bound{}]} \rob(\varphi, \trace{}{}, k, \bound{}) \\
  \rob(\eventually{} \varphi, \trace{}{}, i, \bound{}) & = & \textbf{sup}_{k \in [i, \bound{}]} \rob(\varphi, \trace{}{}, k, \bound{})
\end{eqnarray*}
where \textbf{inf} and \textbf{sup} are the infimum and supremum operators, respectively. Note that the computation of the satisfaction score is restricted to a subsequence of $\trace{}{}$ between $i$ to $\bound{}$ (i.e., $z_i, z_{i+1},...., z_{\bound{}})$. The bound \bound{} may be determined by, for example, the planning horizon used by a planner (i.e., the length of action sequence used by the planner for behavioral prediction). For brevity, the definition for the until (\until{}) operator is omitted, but similar to the one described in \citet{robustness}. 
\end{definition}

\section{Semantic Correctness of ETL Monitors}\label{appx:semantic_corr}

To compare ETL monitors  to ground-truth executions, we define a ground-truth monitor over finite executions.

\begin{definition}[Finite Execution] \label{def:exec_finite}
    For a given execution $\varsigma$, a \emph{finite execution} $\varsigma_{\le t}$ is an execution  consisting of the first $t+1$ states  of $\varsigma$.
\end{definition}
We refer to a finite execution as an execution  when  clear from context.

\begin{definition}[Ground-Truth Monitor]
\label{def:gt_monitor}
Let $\omega$ be a temporal logic formula over \emph{symbolic} predicates, and let $\varsigma_{\le t}$ be a finite execution prefix. The \emph{ground-truth monitor} for $\omega$ on $\varsigma_{\le t}$ is the binary-valued trace
\[
GT_\omega(\varsigma_{\le t})
=
\bigl(
b^\omega_0,\,
b^\omega_1,\,
\dots,\,
b^\omega_t
\bigr)
\in \{-1,+1\}^{t+1},
\]
where, for each prefix $\varsigma_{\le i}$ with $0 \le i \le t$,
\[
b^\omega_i
=
\begin{cases}
+1 & \text{if } \varsigma_{\le i} \models \omega,\\
-1 & \text{otherwise.}
\end{cases}
\]
Here, satisfaction $\varsigma_{\le i}  \models \omega$  is defined by the semantics of the underlying temporal logic over symbolic predicates.
\end{definition}

We define the semantic correctness of monitor $M$ using the following definitions. 
\begin{definition}[Semantic Correspondence]
\label{def:semantic_correspondence}
Let $M$ be an embedding temporal structure with representation map
$\eta = \psi_{\mathrm{enc}} \circ \phi_{\mathrm{obs}}$. Let $\omega$ be a
specification over ground-truth state trajectories $\varsigma_{\le t}$ and let $\varphi$ be an ETL
specification over finite embedding traces $\sigma_{\le t}$. We say that $\varphi$
\emph{corresponds semantically} to $\omega$ with respect to $M$ if, for every
finite execution prefix $\varsigma_{\le t} = (s_0,\dots,s_t)$, letting
$
\sigma_{\le t} = (\eta(s_0),\dots,\eta(s_t)),
$
we have
$
\sigma_{\le t} \models \varphi
\iff
\varsigma_{\le t} \models \omega.
$
\end{definition}

Intuitively, we say that $\varphi$ is a latent space specification \textit{semantically corresponds} 
with the
ground-truth specification $\omega$ if both are intended to capture the same
behavioral property, with $\omega$ defined over state-based predicates and
$\varphi$ defined over embedding-space predicates.
\begin{definition}[Semantic Correctness of an ETL Monitor]
\label{def:semantic_correctness}
Let $\omega$ be a ground-truth specification and let $\varphi$ be an ETL
specification that corresponds semantically to $\omega$. 
An ETL monitor $M$ is \emph{semantically correct} with respect to
$(\varphi,\omega)$ if, for every execution prefix,
$
M_\varphi(\sigma_{\le t}) = GT_\omega(\varsigma_{\le t}).
$
\end{definition}

\section{Formal Definitions of Threshold Calibration}
\label{sec:appendix_threshold}
We provide additional details and intuitive understanding of thresholds and their calibration for ETL monitoring.

Intuitively, $\epsilon$ defines the \emph{boundary of a concept} in the representation space (e.g., how close an observation must be to a goal embedding to count as ``reached''). Unlike state-based predicates, thresholds in embedding spaces depend on the geometry induced by the encoder and the choice of distance function. Poorly chosen thresholds can lead to high false negatives (overly strict predicates) or false positives (overly permissive predicates), making calibration essential.

We propose a data-driven approach to calibrate thresholds for ETL predicates. 
We assume access to a dataset of trajectories (e.g., simulator rollouts or demonstrations) 
with ground-truth signals  $\varsigma_{\le t}$ that satisfy   
a TL specification $\omega$ at timestep $t$ in the underlying state space. Here, $\omega$ is a TL specification that semantically  corresponds to the ETL specification
being instantiated. Then, each trajectory is mapped to an embedding trace using the representation map $\eta$, yielding embeddings $z_t = \eta(s_t)$. From these embeddings, we compute distances to the target set as 
$d_t = d_u(z_t, T_u)$. We then select thresholds using a held-out calibration set of size $n_{\text{cal}}$. 

\subsection{F1-Optimal Threshold}

\begin{definition}[F1-Optimal Threshold]
\label{def:thresh_f1}
Let $\mathcal{D}_{\mathrm{cal}}=\{(d_t, y_t)\}_{t=1}^{N}$ be the calibration set, where $d_t \in \mathbb{R}_{\ge 0}$ is the embedding distance at timestep $t$ and $y_t \in \{0,1\}$ is the corresponding ground-truth predicate label. For any candidate threshold $\epsilon \in \mathbb{R}_{\ge 0}$, define the induced prediction 
\[
\hat{y}_t(\epsilon) = \mathbf{1}[d_t \le \epsilon].
\]
Let
\[
\mathrm{F1}(\epsilon)
=
\frac{2 \, \mathrm{TP}(\epsilon)}
{2 \, \mathrm{TP}(\epsilon) + \mathrm{FP}(\epsilon) + \mathrm{FN}(\epsilon)},
\]
where $\mathrm{TP}(\epsilon)$, $\mathrm{FP}(\epsilon)$, and $\mathrm{FN}(\epsilon)$ denote the numbers of true positives, false positives, and false negatives, respectively, obtained by comparing $\hat{y}_t(\epsilon)$ against $y_t$ over $\mathcal{D}_{\mathrm{cal}}$.
 
The \emph{F1-optimal threshold} is defined as
\[
\epsilon_{F1} \in \arg\max_{\epsilon \in \{d_1,\dots,d_N\}} \mathrm{F1}(\epsilon).
\]
If multiple thresholds achieve the same maximum F1 score, ties may be broken arbitrarily.
\end{definition}

\subsection{Conformal Threshold with Recall Guarantee}

In this work, we propose conformal \emph{ETL predicate calibration} by adapting conformal prediction to select predicate thresholds from embedding distance scores. 
Here, the goal is not to produce a prediction interval around a model output, but to calibrate the threshold used for ETL predicate evaluation.
We do this in two stages. Within each calibration trajectory, we compute a single score from all satisfying timesteps by taking the largest embedding distance among the frames that satisfy the predicate. 
This score represents the \emph{hardest positive case} in that trajectory: any threshold smaller 
than this value would miss at least one true positive frame in that same trajectory.
Then across trajectories, we collect one such score from each calibration trajectory and apply split conformal prediction to these held-out scores. The resulting threshold $\epsilon_{\mathrm{CP}}$ is therefore calibrated not to individual frames, but between different trajectories. As a result, for a new trajectory drawn from the same distribution, with probability at least $1-\alpha$,
the threshold is large enough to cover that trajectory’s hardest positive case, and hence will detect every ground-truth positive timestep in that trajectory. Here, $\alpha \in (0,1)$ is a user-specified error level.

\paragraph{Setup}
Let $\varsigma^1,\ldots,\varsigma^{n_{\mathrm{cal}}}$ be $n_{\mathrm{cal}}$ calibration demonstrations, disjoint from training,
and let $z_g$ denote the target latent embedding that semantically corresponds to a ground-truth specification $\omega$ (see Definition \ref{def:semantic_correspondence}). 
For each calibration demonstration $\varsigma^i$, define the nonconformity score

\[
score_i
= \max_{t \in [0, |\varsigma^i|]} \begin{cases}
d(z_t^{\!\left(i\right)}, z_g) & \mathrm{if }\; \delta^{GT}_{\omega}(s_t^{(i)})=1  \\
-\infty & \mathrm{otherwise.} \\
\end{cases}
\]
for a distance function $d \in D_\embeddings{}$ and where  $\delta^{GT}_{\omega}(s_t^{(i)})=1$ indicates that the ground-truth state at time $t$ satisfies predicate $\omega$. Intuitively, $score_i$ is the most difficult positive timestep in demonstration $\varsigma_i$: any threshold below $score_i$ would fail to classify at least one ground-truth positive frame in that demonstration. Each calibration  demonstration contains at least one ground-truth positive timestep.

\begin{definition}[Conformal Calibration Threshold]
For a calibration set of size $n_{\mathrm{cal}}$ and a target mis-recall level $\alpha \in (0,1)$, a \emph{conformal calibration threshold} $\epsilon_{\mathrm{CP}}$ is a value in $\mathbb{R}_{\ge 0}$ defined as follows:
\[
\epsilon_{\mathrm{CP}} = score_{(k)},
\qquad
k = \left\lceil (1-\alpha)(n_{\mathrm{cal}}+1)\right\rceil,
\]
where $score_{(k)}$ denotes the $k$-th smallest value among
$\{score_i\}_{i=1}^{n_{\mathrm{cal}}}$.
\end{definition}

\begin{theorem}[Conformal Recall Guarantee]
\label{thm:cp}
Assume the calibration demonstrations and a future test demonstration are exchangeable. Then
\[
\mathbb{P}\!\left(
\forall t:\;
\delta^{GT}_{\omega}(s_t^{(i)})=1
\;\Rightarrow\;
d(z_t,z_\omega)\le \epsilon_{\mathrm{CP}}
\right)
\ge 1-\alpha.
\]
Equivalently, with probability at least $1-\alpha$ over a newly drawn test demonstration, the embedding predicate achieves perfect per-demonstration recall, i.e., it detects every ground-truth positive timestep.
\end{theorem}

\begin{proof}[Proof sketch]
By exchangeability, the augmented set of scores
\[
score_1,\ldots,score_{n_{\mathrm{cal}}}, score_{\mathrm{test}}
\]
is exchangeable, where
\[
score_{\mathrm{test}}
=
\max_{\substack{t \in \varsigma_{\mathrm{test}}\\ \delta^{GT}_{\omega}(s_t)=1}}
d\!\left(z_t, z_\omega\right).
\]
The standard split conformal order-statistic argument implies
\[
\mathbb{P}\!\left(score_{\mathrm{test}} \le score_{(k)}\right)
\ge
\frac{k}{n_{\mathrm{cal}}+1}
\ge
1-\alpha.
\]
By construction, the event $score_{\mathrm{test}} \le \epsilon_{\mathrm{CP}}$ is exactly the event that every positive timestep in the test demonstration satisfies $d(z_t,z_\omega)\le \epsilon_{\mathrm{CP}}$, which proves the recall guarantee.
\end{proof}

At deployment, the predicate is evaluated solely by checking whether
\[
d(z_t, z_\omega) \le \epsilon_{\mathrm{CP}},
\]
 The guarantee holds for any embedding function and any underlying data distribution, provided the calibration and test demonstrations remain exchangeable. 

\section{Additional Evaluation Details for Dubins Car}

\subsection{Details for Dubins Car Setup}\label{sec:appendix_dubins_setup}
The navigation task is defined for a robot that respects discrete-time Dubins car dynamics. A state is defined as $s = [p^x, p^y, \theta]$ where $s_{t+1} = s_t + \Delta t [v \cos(\theta_t), v\sin(\theta_t),a_t]$ with continuous angular velocity $a_t \in A = [-a_{\max}, a_{\max}]$.
We fix $a_{\max} = 1.25$ rad/s and $v = 1$ m/s, with discretization at $\Delta t = 0.05 s$.
We generate $N=100$ trajectories using a feedback controller with obstacle avoidance. 
For each task, we construct pairs of equivalent specifications in the ground-truth state space and the embedding space.
The ground-truth specifications, $\omega_i$, are defined over the state space of the car (e.g., the position $p^x, p^y$, angular velocity $y$, etc.).
In parallel, the ETL specifications, $\varphi_i$, are defined over the embedded images observed during simulation; target embeddings are obtained from the simulated observations corresponding to selected goal states.
We utilize the encoder from the world model used in \citet{anysafe}, which is based on Dreamer \citep{Hafner2025} with a Recurrent State Space Model (RSSM). The encoder produces a task-relevant representation space in which distances reflect semantic similarity between observations. 

We evaluate four specification patterns: \emph{Reach} ($\eventually A$), \emph{Avoid} ($\always \neg C$), \emph{Reach-Avoid} ($\eventually A \land \always \neg C$), and \emph{Sequential} ($\eventually(A \land \eventually B)$). $A$, $B$, and $C$ denote reaching the top-right goal $|p_t - (0.8, 0.8)| < 0.25$, reaching the top-left goal $|p_t - (-0.8, 0.8)| < 0.25$, and entering the obstacle proximity zone $|p_t - (0.8, -0.8)| < 0.5$, respectively. This setup enables direct comparison between embedding-based satisfaction and ground-truth logical semantics. Thresholds are calibrated on a 40/60 calibration/test split using both F1-optimal and conformal  procedures. In all experiments, we use $\alpha = 0.10$ for computing $\epsilon_{cp}$.

\subsection{Results for Dubins Car}\label{sec:appendix_dubins}

The results for the experiments with the Dubins car can be found in Table \ref{tab:dubins}, including F1, Precision, Recall, and Agreement scores for both F1-Optimal and Conformal Prediction thresholds.

\begin{table}[t]
\centering
\tabletextsize
\caption{ETL predicate evaluation on the Dubins car.
}
\label{tab:dubins}
\begin{tabular}{llcccc|ccc}
\toprule
 & & \multicolumn{4}{c|}{$\epsilon^*$ (F1-optimal)} & \multicolumn{3}{c}{$\epsilon_{CP}$ ($\alpha{=}0.10$)} \\
Spec & Scope & F1 & Prec. & Rec. & Agree & Prec. & Rec. & Agree \\
\midrule
Reach $A$ & frames & 0.85 & 0.87 & 0.83 & 98.6\% & 0.79 & 0.93 & 98.5\% \\
Avoid $C$ & frames & 0.80 & 0.69 & 0.94 & 96.5\% & 0.70 & 0.92 & 96.5\% \\
RA \ $A\wedge\lnot C$ & frames & 0.85 & 0.87 & 0.83 & 98.6\% & 0.79 & 0.93 & 98.5\% \\
Seq.\ $A \to B$ & episodes & 1.00 & 1.00 & 1.00 & 100\% & 1.00 & 1.00 & 100\% \\
\bottomrule
\end{tabular}%
\end{table}

\section{Additional Details for Manipulation Tasks}
\label{sec:appendix_manip}

We provide qualitative visualizations of ETL predicate behavior on manipulation tasks. These plots illustrate how embedding distances evolve over time and how predicate satisfaction aligns with ground-truth signals.

\subsection{Tabularized Results}\label{appx:tab:manip}
We present tabularized results from the experiments from Section~\ref{sec:manip} in Table~\ref{tab:manip} and Table~\ref{tab:f1_three_envs}.

\begin{table}[t]
\centering
\tabletextsize
\caption{ETL predicate evaluation on  
MetaWorld (test split). Frame-level metrics are reported for all predicates.
}
\label{tab:manip}
\begin{tabular}{
lcccc|cc}
\toprule
Task & F1 & Prec. & Rec. & Agree & CP Rec. & CP Prec. \\
\midrule
Grasp ($A$) & 0.985 & 0.971 & 0.999 & 97.3\% & 0.883 & 0.976 \\
Place ($B$) & 0.948 & 0.901 & 1.000 & 91.4\% & 0.916 & 0.926 \\
\midrule
\multicolumn{2}{l}{Sequential}
& \multicolumn{4}{c}{\textbf{100\% ordering agreement (24/24)}} \\
\bottomrule
\end{tabular}
\end{table}
\begin{table}[t]
\centering
\caption{F1 score for failure / predicate detection across three environments. All methods use the F1-optimal threshold; ETL (CP) additionally reports the class-conditional split-CP threshold. Per-column best in \textbf{bold}.}
\label{tab:f1_three_envs}
\begin{tabular}{lcccc}
\toprule
Method & \textsc{Stacking} & \textsc{Sorting} & \textsc{MW-Pick-Place} & Avg \\
\midrule
logpZO & 0.844 & 0.534 & \textbf{0.992} & 0.790 \\
PCA-kmeans & 0.852 & 0.585 & 0.815 & 0.751 \\
ETL (best F1) & \textbf{0.897} & \textbf{0.593} & 0.961 & \textbf{0.817} \\
ETL (CP) & 0.887 & 0.556 & 0.923 & 0.789 \\
\bottomrule
\end{tabular}
\end{table}

\subsection{Sequential Predicate Trace}
\label{sec:appendix_metaworld}

Figure~\ref{fig:mw_seq_timelines} illustrates dual-predicate traces for the pick-place-wall task. The plots show distances to both subgoal embeddings ($z_A$ for grasp and $z_B$ for place), along with corresponding predicate activations.

\begin{figure}[t]
  \centering
  \includegraphics[width=0.9\textwidth]{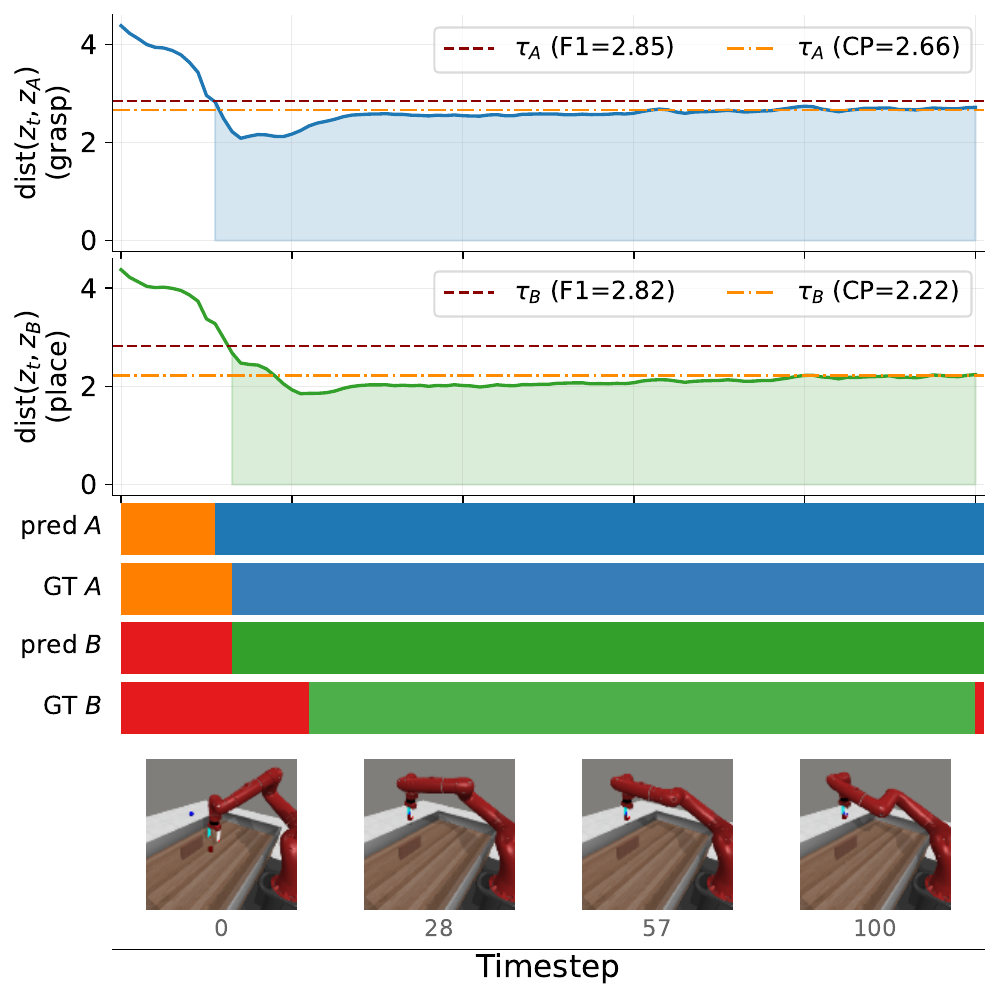}
  \caption{Dual-predicate Boolean timelines for \texttt{mw-pick-place-wall}. Top panels show distances to grasp ($z_A$) and place ($z_B$) embeddings with thresholds $\epsilon^*$ and $\epsilon_{CP}$. Lower panels show ETL and ground-truth predicate activations. }
  \label{fig:mw_seq_timelines}
\end{figure}

The embedding distances to $z_A$ and $z_B$ decrease at distinct timesteps corresponding to grasp and placement events. This temporal separation is preserved in the embedding space, allowing ETL to correctly identify the ordering of subgoals across all trajectories.

\section{Additional DROID Evaluation Details}
\label{app:droid_full}

\subsection{Tabularized Results}\label{sec:appendix_droid}
We present tabularized results from the experiments from Section~\ref{sec:eval:droid} in Table~\ref{tab:droid_phasic}.

\begin{table}
\centering
\tabletextsize
\caption{Comparison of ETL and Qwen2-VL-2B on phasic DROID episodes.}
\label{tab:droid_phasic}
\setlength{\tabcolsep}{2.8pt}
\begin{tabular}{l c ccc ccc}
\toprule
& & \multicolumn{3}{c}{\textbf{ETL}} & \multicolumn{3}{c}{\textbf{Qwen2-VL-2B}} \\
\cmidrule(lr){3-5} \cmidrule(lr){6-8}
\textbf{Task} & \textbf{Ph.} & \textbf{F1} & \textbf{Agr.} & \textbf{Seq.} & \textbf{F1} & \textbf{Agr.} & \textbf{Seq.} \\
\midrule
Obj. $\to$ paper      & 2 & 0.968 & 0.977 & \cmark & 0.493 & 0.383 & \cmark \\
Tap $\to$ spoon       & 2 & 0.839 & 0.937 & \xmark & 0.292 & 0.700 & \cmark \\
4-way chain           & 4 & 0.531 & 0.862 & \cmark & 0.585 & 0.811 & \xmark \\
White $\to$ blue cloth& 3 & 0.884 & 0.959 & \cmark & 0.335 & 0.258 & \xmark \\
Bottles $\to$ stove   & 3 & 0.844 & 0.966 & \cmark & 0.244 & 0.685 & \xmark \\
\midrule
\textbf{Mean} & -- & \textbf{0.813} & \textbf{0.940} & \textbf{4/5} & \textbf{0.390} & \textbf{0.567} & \textbf{2/5} \\
\bottomrule
\end{tabular}
\end{table}

\subsection{Ground Truth Construction}

Unlike simulation environments, DROID does not provide explicit task-phase annotations. We derive ground-truth predicates directly from proprioceptive signals:
\begin{align*}
  \pi_{\text{hold}}(t)    &: \text{gripper}(t) > 0.5 \\
  \pi_{\text{release}}(t) &: \text{gripper}(t) < 0.15 \;\wedge\; \exists\,t' < t,\, \pi_{\text{hold}}(t')
\end{align*}
The sequential specification is defined as:
\[
\exists\,t_1 < t_2 : \pi_{\text{hold}}(t_1) \wedge \pi_{\text{release}}(t_2),
\]
corresponding to a pick-and-place interaction. For tasks where proprioception is not sufficient to identify phase of tasks, we manually annotate the videos and generate the ground truth predicates.

\subsection{Encoder and Representation}

We use the \textbf{SVD VAE} from Ctrl-World \citep{ctrlworld}, a video diffusion model trained directly on DROID data. Each frame is encoded into a $4 \times 24 \times 40$ latent tensor, flattened to a 3{,}840-dimensional vector. Distances are computed using cosine similarity. This encoder is domain-matched to DROID and provides stronger geometric separation than general-purpose encoders such as DINOv2.

\subsection{Sequential Evaluation Details}

We identify 25 episodes containing valid grasp-then-release transitions, using 10 for calibration and 15 for evaluation.

\begin{table}[h]
\centering
\tabletextsize
\caption{Sequential predicate evaluation on DROID (SVD VAE, wrist camera).}
\label{tab:droid_seq_app}
\begin{tabular}{lccccc}
\toprule
Predicate & F1 & Precision & Recall & Agreement & Seq.\ Agreement \\
\midrule
$\pi_{\text{hold}}$   & 0.666 & 0.721 & 0.750 & 0.701 & \multirow{2}{*}{\textbf{1.000}} \\
$\pi_{\text{release}}$ & 0.258 & 0.162 & 1.000 & 0.163 & \\
\bottomrule
\end{tabular}
\end{table}

The release predicate exhibits low precision due to visual ambiguity: the approach phase with an open gripper is visually similar to the post-release phase. Despite this, sequential ordering is correctly identified in all episodes.
\subsection{Discussion of Failure Modes}
\label{sec:appendix_manip_failures}

ETL performance depends on the separability of task-relevant concepts in the embedding space. For manipulation tasks, success states (e.g., button press, object placement) are highly distinctive, leading to sharp transitions in embedding distance and high predicate accuracy.

Failure modes arise when embeddings corresponding to different semantic states overlap in the latent space, which can result in delayed or premature predicate activation near decision boundaries. However, such cases are rare in these tasks, and conformal thresholds help mitigate missed detections by prioritizing recall.

Overall, these qualitative results reinforce that embedding distances provide a reliable and interpretable signal for detecting semantic events and temporal structure in manipulation tasks.

We evaluate whether \emph{Embedding Temporal Logic} (ETL) monitoring can serve
as a competitive, interpretable alternative to state of the art failure
predictors for generative robot policies.  Our central hypothesis is that
explicitly encoding the \emph{sequential milestone structure} of a task into a
small number of latent spec anchors, and then asking whether each anchor was
ever approached during execution, yields a sharper failure signal than methods
that either model the full observation distribution or score each timestep
independently against a single goal embedding.

\end{document}